\documentclass{article} 
\usepackage{nips15submit_e,times}
\usepackage{hyperref}
\usepackage{url}
\usepackage{amsfonts, amsmath, amssymb, amsthm, constants,comment}
\usepackage{graphicx}
\usepackage{epsfig, multirow, placeins}
\usepackage{algorithm}
\usepackage{algorithmic}

\usepackage{tikz}
\usepackage{tikz}
\usepackage{hyperref}
\hypersetup{colorlinks=true, linkcolor=red, citecolor=blue, urlcolor=darkgreen}

\usetikzlibrary{arrows,calc}
\usetikzlibrary{mindmap,trees,backgrounds}
\usetikzlibrary{positioning}
\usepackage{verbatim}
\usepackage{caption}
\usepackage{subcaption}
\usepackage{bbm}
\newtheorem{theorem}{Theorem}

\newtheorem{lemma}{Lemma}

\newtheorem{remark}{Remark}

\newtheorem{corollary}{Corollary}

\providecommand{\nor}[1]{\left\lVert {#1} \right\rVert}

\providecommand{\scalT}[2]{\left\langle{#1},{#2}\right\rangle}

\newcommand{\EE}{\mathcal E}

\def\argmin{\operatornamewithlimits{arg\,min}}

\newcommand{\XX}{\mathcal X}

\newcommand{\YY}{\mathcal Y}

\newcommand{\ind}{1{\hskip -2.5 pt}\hbox{I}}

\newcommand{\eps}{\varepsilon}

\title{Learning with Group Invariant Features:\\ A Kernel Perspective.}

\author{
Youssef Mroueh \\
Multimodal Algorithms \& Engines Group\\
IBM T.J Watson Reseach Center\\
\texttt{mroueh@us.ibm.com} \\
\And
Stephen Voinea$^{*}$ \\
CBMM, MIT. \\
\texttt{voinea@mit.edu} \\
$*$\emph{Co-first author}
\And
Tomaso Poggio \\
CBMM, MIT . \\
\texttt{tp@ai.mit.edu} \\
}

\nipsfinalcopy 

\begin{document}

\maketitle

\begin{abstract}
We analyze in this paper a random feature map based on a  theory of invariance (\emph{I-theory}) introduced in \cite{AnselmiLRMTP13}. More specifically, a group invariant  signal signature is obtained through cumulative distributions of group-transformed random projections. Our analysis bridges invariant feature learning with kernel methods, as we show that this feature map defines an expected Haar-integration kernel that is invariant to the specified group action. We show how this non-linear random feature map approximates this group invariant kernel uniformly on a  set of $N$ points. Moreover, we show that it defines a function space that is dense in the equivalent Invariant Reproducing Kernel Hilbert Space. Finally, we quantify error rates of the convergence of the empirical risk minimization, as well as the reduction in the sample complexity of a learning algorithm using such an invariant representation for signal classification, in a classical supervised learning setting.
\end{abstract}
\section{Introduction }
Encoding signals or building similarity kernels that are invariant to the action of a group is a key problem in unsupervised learning, as it reduces the complexity of the learning task and mimics how our brain represents information invariantly to symmetries and various nuisance factors (change in lighting in image classification and pitch variation in speech recognition)  \cite{AnselmiLRMTP13,scat,HintonAuto,BengioCV13}. 
Convolutional neural networks \cite{lecun,KrizhevskySH12} achieve state of the art performance in many computer vision and speech recognition tasks, but require a large amount of labeled examples as well as augmented data, where we reflect symmetries of the world through virtual examples \cite{Niyogi98incorporatingprior, yaser} obtained by applying identity-preserving transformations such as shearing, rotation, translation, etc., to the training data.
In this work, we adopt the approach of \cite{AnselmiLRMTP13}, where the representation of the signal is designed to reflect the invariant properties and model the world symmetries with group actions. The ultimate aim is to bridge unsupervised learning of invariant representations with invariant kernel methods, where we can use tools from classical supervised learning to easily address the statistical consistency and sample complexity questions \cite {vapnik98,stch08}.   
Indeed, many invariant kernel methods and related invariant kernel networks have been proposed. We refer the reader to the related work section for a review (Section \ref{sec:prevwork}) and we start by showing how to accomplish this invariance through group-invariant Haar-integration kernels \cite{HaasdonkVB05}, and then show how random features derived from a memory-based theory of invariances introduced in \cite{AnselmiLRMTP13} approximate such a kernel.
\subsection{Group Invariant Kernels} We start by reviewing group-invariant Haar-integration kernels introduced in \cite{HaasdonkVB05}, and their use in a binary classification problem. This section highlights the conceptual advantages of such kernels as well as their practical inconvenience, putting into perspective the advantage of approximating them with explicit and invariant random feature maps.

\noindent \textbf{Invariant Haar-Integration Kernels.}
We consider a subset $\mathcal{X}$ of the hypersphere in $d$ dimensions  $\mathbb{S}^{d-1}$. Let $\rho_{\XX}$ be a measure on $\mathcal{X}$.
Consider a kernel $k_{0}$ on $\XX$, such as a radial basis function kernel. Let $G$ be a group acting on $\XX$, with a normalized Haar measure $\mu$. $G$ is assumed to be a compact and unitary group.
Define an invariant kernel $\mathcal{K}$ between $x,z \in \XX$ through Haar-integration \cite{HaasdonkVB05} as follows:
\begin{equation}\label{eq:haarInv}
\mathcal{K}(x,z)=\int_{G} \int_{G} k_{0}(gx,g'z) d\mu(g)d\mu(g').
\end{equation}
As we are integrating over the entire group, it is easy to see that:
$\mathcal{K}(g'x,gz)=\mathcal{K}(x,z), ~ \forall g,g' \in G, \forall x,z \in \XX.$
Hence the Haar-integration kernel is invariant to the group action. The symmetry of $\mathcal{K}$ is obvious. Moreover, if $k_0$ is a positive definite kernel, it follows that $\mathcal{K}$ is positive definite as well \cite{HaasdonkVB05}. One can see the Haar-integration kernel framework as another form of data augmentation, since we have to produce group-transformed points in order to compute the kernel.

\noindent \textbf{Invariant Decision Boundary.} Turning now to a binary classification problem, we assume that we are given a labeled training set:
$S=\{(x_i,y_i)\mid x_i \in \XX, y_i \in \YY= \{\pm1\} \}_{i=1}^N.$
In order to learn a decision function $f:\XX\to \YY$, we minimize the following empirical risk induced by an $L$-Lipschitz, convex loss function $V$, with $V'(0)<0$ \cite{bajomc06}:
$\min_{f\in \mathcal{H}_{\mathcal{K}}}\hat{\mathcal{E}}_{V}(f):=\frac{1}{N}\sum_{i=1}^N V(y_if(x_i)) $, where we restrict $f$ to belong to a hypothesis class  induced by the invariant kernel $\mathcal{K}$, the so called Reproducing Kernel Hilbert Space $\mathcal{H}_{\mathcal{K}}$. 
The representer theorem \cite{wahba90} shows that the solution of such a problem, or the optimal decision boundary $f^*_{N}$ has the following form:
$f^*_{N}(x)=\sum_{i=1}^N \alpha^*_i\mathcal{K}(x,x_i).$
Since the kernel $\mathcal{K}$ is group-invariant it follows that :
$f^*_{N}(gx)= \sum_{i=1}^N \alpha_i \mathcal{K}(gx,x_i)=  \sum_{i=1}^N \alpha_i \mathcal{K}(x,x_i)=f^*_{N}(x), ~ \forall g \in G.$
Hence the   the decision boundary  $f^*$is group-invariant as well, and we have:
$f^*_{N}(gx)=f^*_{N}(x), \forall g \in G, \forall x \in \XX.$

\noindent \textbf{Reduced Sample Complexity.}
We have shown that a group-invariant kernel induces a group-invariant decision boundary, but how does this translate to the sample complexity of the learning algorithm?
To answer this question, we will assume that the input set $\XX$ has the following structure:
$\XX=\XX_{0}\cup \mathcal{G}\XX_{0}, ~ \mathcal{G}\XX_{0}=\{z | z= gx, x\in \XX_{0}, g\in G /\ \{e\} \}, $
where $e$ is the identity group element.\\
This structure implies that for a function $f$ in the invariant RKHS $\mathcal{H}_{\mathcal{K}}$, we have:
$$ \forall z \in \mathcal{G}\mathcal{X}_0, \exists ~x \in \mathcal{X}_0, \exists~ g \in G \text{ such that, } z=gx, \text{ and } f(z)=f(x).$$
Let $\rho_{y}(x)=\mathbb{P}(Y=y|x)$ be the label posteriors. We assume that $\rho_{y}(gx)=\rho_{y}(x) , \forall g\in G$. This is a natural assumption since the label is unchanged given the group action. Assume that the set $\XX$ is endowed with a measure $\rho_{\XX}$ that is also group-invariant.
Let $f$ be the group-invariant decision function and consider the expected risk induced by the loss $V$, $\mathcal{E}_{V}(f)$, defined as follows:
 \begin{equation}\label{eq:risk}
\mathcal{E}_{V}(f)=\int_{\XX} \sum_{y\in \YY}V(yf(x))\rho_{y}(x)\rho_{\XX}(x) dx,
\end{equation}
$\mathcal{E}_{V}(f)$  is a proxy to the misclassification risk \cite{bajomc06}.
Using the invariant properties of the function class and the data distribution we have by invariance of $f$, $\rho_y$, and $\rho$: 
\begin{eqnarray*}
\mathcal{E}_{V}(f)
&=&\int_{\XX_0} \sum_{y\in \YY}V(yf(x))\rho_{y}(x)\rho_{\XX}(x)dx + \int_{\mathcal{G}\XX_0} \sum_{y\in \YY}V(yf(z))\rho_{y}(z)\rho_{\XX}(z)dz\\
&=&\int_{G} d\mu(g)\int_{\XX_0} \sum_{y\in \YY} V(y f(gx))\rho_{y}(gx)\rho_{\XX}(x) dx\\
&=&\int_{G} d\mu(g)\int_{\XX_0} \sum_{y\in \YY} V(y f(x))\rho_{y}(x)\rho_{\XX}(x) dx \text{~~(By invariance of $f$, $\rho_y$, and $\rho$ )}\\
&=&\int_{\XX_0} \sum_{y\in \YY} V(y f(x))\rho_{y}(x)\rho_{\XX}(x) dx.
\end{eqnarray*}
Hence, given an invariant kernel to a group action that is identity preserving, it is sufficient to minimize the empirical risk on the core set $\XX_{0}$, and it generalizes to samples in $\mathcal{G}\XX_0$.\\
Let us imagine that $\XX$ is finite with cardinality $|\XX|$; the cardinality of the core set $\XX_{0}$ is a small fraction of the cardinality of $\XX$: $|\XX_0|=\alpha|\XX|,$ where $0<\alpha<1$. Hence, when we sample training points from $\XX_{0}$, the maximum size of the training set is $N=\alpha |\XX|<<|\XX|$, yielding a reduction in the sample complexity.

%
%
\subsection{Contributions}
We have just reviewed the group-invariant Haar-integration kernel. 
In summary, a group-invariant kernel implies the existence of a decision function that is invariant to the group action, as well as a reduction in the sample complexity due to sampling training points from a reduced set, a.k.a the core set $\XX_{0}$.\\
Kernel methods with Haar-integration kernels come at a very expensive computational price at both training and test time: computing the Kernel is computationally cumbersome as we have to integrate over the group and produce virtual examples by transforming points explicitly through the group action. Moreover, the training complexity of kernel methods scales cubicly in the sample size. Those practical considerations make the usefulness of such kernels very limited.\\
The contributions of this paper are on three folds:
\begin{enumerate}
\item We first show that a non-linear random feature map $\Phi:\XX \to \mathbb{R}^D$  derived from a memory-based theory of invariances introduced in \cite{AnselmiLRMTP13} induces an expected group-invariant Haar-integration kernel $K$. For fixed points $x,z\in \XX$, we have:
$\mathbb{E}\scalT{\Phi(x)}{\Phi(z)}= K(x,z),$
where $K$ satisfies: $K(gx,g'z)=K(x,z), \forall g,g' \in G, x,z \in \XX.$
\item  We show a Johnson-Lindenstrauss type result that holds uniformly on a set of $N$ points that assess the concentration of this random feature map around its expected induced kernel. For sufficiently large $D$, we have $\scalT{\Phi(x)}{\Phi(z)}\approx K(x,z)$, uniformly on an $N$ points set.
\item We show that, with a linear model, an invariant decision function can be learned in this random feature space by sampling points from the core set $\XX_0$ i.e:
$f^*_{N}(x)\approx \scalT{w^*}{\Phi(x)}$
and generalizes to unseen points in $\mathcal{G}\XX_0$, reducing the sample complexity. Moreover, we show that those features define a function space that approximates a dense subset of the invariant RKHS, and assess the error rates of the empirical risk minimization
using such random features.  
\item We demonstrate the validity of these claims on three datasets: text (artificial), vision (MNIST), and speech (TIDIGITS).
\end{enumerate}

\section{From Group Invariant Kernels to Feature Maps }
In this paper we show that a random feature map based on I-theory \cite{AnselmiLRMTP13}: $\Phi: \mathcal{X}\to \mathbb{R}^{D}$ approximates  a group-invariant Haar-integration kernel $K$ having the form given in Equation \eqref{eq:haarInv}: 
$$\scalT{\Phi(x)}{\Phi(z)}\approx K(x,z).$$  
We start with some notation that will be useful for defining the feature map.
Denote the cumulative distribution function of a random variable $X$ by,
$$F_{X}(\tau)=\mathbb{P}(X\leq \tau),$$
Fix $x\in\XX$, Let $g \in G$ be a random variable drawn according to the normalized Haar measure $\mu$ and let $t$ be a random template whose distribution will be defined later. For $s>0$, define the following truncated cumulative distribution function (CDF) of the dot product $\scalT{x}{gt}$:
$$\psi(x,t,\tau)= \mathbb{P}_{g}(\scalT{x}{gt}\leq \tau)=  F_{\scalT{x}{gt}}(\tau), ~\tau \in [-s,s] , x\in \XX,  $$
Let $\eps \in (0,1)$. We consider the following Gaussian vectors (sampling with rejection) for the templates $t$: 
$$t = n \sim \mathcal{N}\left(0, \frac{1}{d} I_d\right) , \text{ if } \nor{n}^2_2 < 1+\eps, ~t=\perp \text{ else }.$$
The reason behind this sampling is to keep the range of $\scalT{x}{gt}$ under control: The squared norm $\nor{n}^2_2 $ will be bounded by $1+\eps$ with high probability by a classical concentration result (See proof of Theorem \ref{theo:CDF_RP} for more details).  The group being unitary and $x \in \mathbb{S}^{d-1}$, we know that : $|\scalT{x}{gt}| \leq \nor{n}_2 < \sqrt{1+\eps} \leq 1+\eps $, for $\eps \in (0,1)$.
\begin{remark}
We can also consider templates $t$, drawn uniformly on the unit sphere $\mathbb{S}^{d-1}$. Uniform templates on the sphere can be drawn as follows:
$$t= \frac{\nu}{\nor{\nu}_2}, ~\nu \sim \mathcal{N}(0,I_d),$$
since the norm of a gaussian vector is highly concentrated around its mean $\sqrt{d}$, we can use the gaussian sampling with rejection. 
Results proved for gaussian templates (with rejection) will hold true for templates drawn at uniform on the sphere with different constants.
\end{remark}
\noindent Define the following kernel function,
\begin{eqnarray*}
K_{s}(x,z)&=& \mathbb{E}_{t} \int_{-s}^s \psi(x,t,\tau) \psi(z,t,\tau)d\tau,
\end{eqnarray*}
where $s$ will be fixed throughout the paper to be $s=1+\eps$ since the gaussian sampling with rejection controls the dot product to be in that range.\\
Let $\bar{g} \in G$. As the group is closed, we have $\psi(t,\bar{g}x,\tau)=\int_{G} \ind_{\scalT{g\bar{g}x}{t}\leq \tau } d\mu(g)=\int_{G} \ind_{\scalT{gx}{t}\leq \tau } d\mu(g)= \psi(t,x,\tau) $ and hence $K_{s}(gx,g'z)= K_{s}(x,z),$ for all $g,g'\in G$. It is clear now that $K$ is a group-invariant kernel.  \\
In order to approximate $K_{s}$, we sample $|G|$ elements uniformly and independently from the group $G$, i.e. $g_i, i=1\dots |G|$, and define the normalized empirical CDF :
$$ \phi (x,t,\tau)= \frac{1}{|G|\sqrt{ m }}\sum_{i=1}^{|G|} \ind_{\scalT{g_it}{x}\leq \tau } , ~-s\leq \tau\leq s. $$
We discretize the continuous threshold $\tau$ as follows:
$$\phi\left(x,t,\frac{sk}{n}\right)=\frac{\sqrt{s}}{\sqrt{n m }|G|}\sum_{i=1}^{|G|} \ind_{\scalT{g_it}{x}\leq \frac{s}{n}k}, ~-n\leq k \leq n.$$
We sample $m$ templates independently according to the Gaussian sampling with rejection, $t_j, j=1\dots m$.
We are now ready to define the random feature map $\Phi$:
$$\Phi(x)=\left[\phi\left(x,t_j,\frac{sk}{n}\right)\right ]_{j=1\dots m, k=-n\dots n}\in \mathbb{R}^{(2n+1)\times m}.$$
It is easy to see that: 
$$\lim _{n\to \infty} \mathbb{E}_{t,g} \scalT{\Phi(x)}{ \Phi(z)}_{\mathbb{R}^{(2n+1)\times m}}=\lim _{n\to \infty} \mathbb{E}_{t,g} \sum_{j=1}^m\sum_{k=-n}^{n}\phi\left(x,t_j,\frac{sk}{n}\right)\phi\left(z,t_j,\frac{sk}{n}\right)=K_{s}(x,z).$$
In Section \ref{sec:geometry} we study the geometric information captured by this kernel by stating explicitly the similarity it computes.
\begin{remark}[Efficiency of the representation]
1) The main advantage of such a feature map, as outlined in \cite{AnselmiLRMTP13}, is that we store transformed templates in order to compute $\Phi$, while if we wanted to compute an invariant kernel of type $\mathcal{K}$ (Equation \eqref{eq:haarInv}), we would need to explicitly transform the points. The latter is computationally expensive. Storing transformed templates and computing the signature $\Phi$ is much more efficient. It falls in the category of  memory-based learning, and is biologically plausible  \cite{AnselmiLRMTP13}. \\
2) As $|G|$,$m$,$n$ get large enough, the feature map $\Phi$ approximates a group-invariant Kernel, as we will see in next section.
\end{remark}

\section{An Equivalent Expected  Kernel and a Uniform Concentration Result}\label{sec:geometry}
In this section we present our main results, with proofs given in the supplementary material . Theorem \ref{theo:CDF_RP} shows that the random feature map $\Phi$, defined in the previous section, corresponds in expectation to a group-invariant Haar-integration kernel $K_{s}(x,z)$. Moreover, $s-K_{s}(x,z)$ computes the average pairwise distance between all points in the orbits of $x$ and $z$, where the orbit is defined as the collection of all group-transformations of a given point $x$ : $\mathcal{O}_{x}=\{gx,g \in G\}$.
\begin{theorem}[Expectation]\label{theo:CDF_RP} 
Let $\eps \in (0,1)$ and  $x,z \in \XX$. Define the distance $d_{G}$ between the orbits $\mathcal{O}_{x}$ and $\mathcal{O}_{z}$:  $$d_{G}(x,z)= \frac{1}{\sqrt{2\pi d}}\int_{G}\int_{G}  \nor{gx-g'z}_2 d\mu(g)d\mu(g'),$$
and the group-invariant expected kernel $$K_{s} (x,z)=\lim _{n\to \infty} \mathbb{E}_{t,g} \scalT{\Phi(x)}{ \Phi(z)}_{\mathbb{R}^{(2n+1)\times m}}=\mathbb{E}_{t} \int_{-s}^{s} \psi(x,t,\tau)\psi(z,t,\tau) d\tau, ~ s=1+\eps. $$
\begin{enumerate}
\item The following inequality holds with probability 1:
 \begin{equation} \label{eq:sandwich}
\eps - \delta_2(d,\eps) \leq K_{s}(x,z) - \left( 1-d_{G}(x,z)\right)\leq \eps+ \delta_{1}(d,\eps), 
 \end{equation}
 where $ \delta_{1}(\eps,d)=\frac{e^{- d \eps^2/16}}{\sqrt{d}} -\frac{1}{2}\frac{e^{-\eps d/2} \left(1+\eps\right)^{\frac{d}{2}}}{\sqrt{d}}$ and $\delta_2(\eps,\delta)=   \frac{e^{- d \eps^2/16}}{\sqrt{d}}+ (1+\eps)e^{-d\eps^2/8} $.
\item For any  $\eps \in (0,1)$  as  the dimension $d \to \infty $ we have   $\delta_{1}(\eps,d) \to 0 $ and $\delta_{2}(\eps,d) \to 0$, and we have asymptotically $K_{s}(x,z) \to 1- d_{G}(x,z) +\eps= s-d_{G}(x,z)$.
\item $K_{s}$ is symmetric and $K_{s}$  is positive semi-definite.

\end{enumerate}
\end{theorem}
\begin{remark}
1) $\eps,\delta_1(d,\eps),$ and $\delta_2(d,\eps)$ are not errors due to results holding with high probability but are due to the truncation and are a technical artifact of the proof. 
2) Local invariance can be defined by restricting the sampling of the group elements to a subset $\mathcal{G} \subset G$. Assuming that for each $g\in \mathcal{G}, g^{-1}\in \mathcal{G}$, the equivalent kernel has asymptotically the following form: 
 $$K_{s}(x,z)\approx s- \frac{1}{\sqrt{2\pi d}}\int_{\mathcal{G}}\int_{\mathcal{G}} \nor{gx-g'z}_2 d\mu(g)d\mu(g')  .$$
 3) The norm-one constraint can be relaxed, let $R= \sup_{x\in \XX} \nor{x}_{2}<\infty$, hence we can set $s=R(1+\eps)$, and \begin{equation} \label{eq:sandwich}
 - \delta_2(d,\eps) \leq K_{s}(x,z) - \left( R(1+\eps)-d_{G}(x,z)\right)\leq  \delta_{1}(d,\eps), 
 \end{equation}
 where $ \delta_{1}(\eps,d)=R\frac{e^{- d \eps^2/16}}{\sqrt{d}} -\frac{R}{2}\frac{e^{-\eps d/2} \left(1+\eps\right)^{\frac{d}{2}}}{\sqrt{d}}$ and $\delta_2(\eps,\delta)=   R\frac{e^{- d \eps^2/16}}{\sqrt{d}}+ R(1+\eps)e^{-d\eps^2/8} $.
\end{remark}
Theorem \ref{theo:InvJL} is, in a sense, an invariant Johnson-Lindenstrauss \cite{JL} type result where we show that the dot product defined by the random feature map $\Phi$
, i.e $\scalT{\Phi(x)}{\Phi(z)}$, is concentrated around the invariant expected kernel uniformly on a data set of $N$ points, given a sufficiently large number of templates $m$, a large number of sampled group elements $|G|$, and a  large bin number $n$. The error naturally decomposes  to a numerical error $\eps_0$ and  statistical errors $\eps_1,\eps_2$ due to the sampling of the templates and the group elements respectively.

\begin{theorem}\label{theo:InvJL}[Johnson-Lindenstrauss type Theorem- $N$ point Set] Let $\mathcal{D}=\{x_i\mid x_i\in \XX\}_{i=1}^N$ be a finite dataset. Fix $\eps_0,\eps_1,\eps_2,\delta_1,\delta_2 \in (0,1)$. For a number of bins 
$n\geq \frac{1}{\eps_0}$, templates $m\geq \frac{C_1 }{\eps^2_1}\log(\frac{N}{\delta_1})$, and group elements $|G|\geq\frac{C_2}{\eps^2_2}\log(\frac{Nm}{\delta_2})$, where $C_1,C_2$ are universal numeric constants, we have:
\begin{equation}\label{eq:JL}
\left|\scalT{\Phi(x_i)}{\Phi(x_j)}-K_{s}(x_i,x_j)\right| \leq  \eps_{0}+\eps_{1}+\eps_{2} , i=1\dots N, j=1\dots N ,
\end{equation}
with probability $1-\delta_1-\delta_2$.
\end{theorem}
Putting together Theorems \ref{theo:CDF_RP} and \ref{theo:InvJL}, the following Corollary shows how the group-invariant random feature map $\Phi$ captures 
the invariant distance between points uniformly on a dataset of $N$ points.
\begin{corollary} [Invariant Features Maps and Distances between Orbits]
Let $\mathcal{D}=\{x_i\mid x_i\in \XX\}_{i=1}^N$ be a finite dataset. Fix $\eps_0,\delta \in (0,1)$. For a number of bins 
$n\geq \frac{3}{\eps_0}$, templates $m\geq \frac{9C_1 }{\eps^2_0}\log(\frac{N}{\delta})$, and group elements $|G|\geq\frac{9C_2}{\eps^2_0}\log(\frac{Nm}{\delta})$, where $C_1,C_2$ are universal numeric constants, we have:
\begin{equation}\label{eq:JL}
\eps - \delta_2(d,\eps)-\eps_0 \leq \scalT{\Phi(x_i)}{\Phi(x_j)}- (1-d_{G}(x_i,x_j))\leq  \eps_{0}+ \eps+ \delta_{1}(d,\eps), 
\end{equation}
$ i=1\dots N, j=1\dots N$, with probability $1-2\delta$.

\end{corollary}

\begin{remark}
 Assuming that the templates are unitary and drawn form a general distribution $p(t)$, the equivalent kernel has the following form: 
 $$K_{s}(x,z)= \int_{\mathcal{G}}\int_{\mathcal{G}}  d\mu(g)d\mu(g')  \left(\int s-\max(\scalT{x}{gt},\scalT{z}{g't}) p(t)dt\right).$$
Indeed when we use the gaussian sampling with rejection for the templates, the integral  $\int \max(\scalT{x}{gt},\scalT{z}{g't}) p(t)dt$ is asymptotically proportional to  $\nor{g^{-1}x-g^{',-1}z}_2$ .
It is interesting to consider different distributions that are domain-specific for the templates and assess the number of the templates needed to approximate such kernels. It is also interesting to find the optimal templates that achieve the minimum distortion in equation \ref{eq:JL}, in a data dependent way, but we will address these points in future work.
\end{remark}

\section{Learning with Group Invariant Random Features}
In this section, we show that learning a linear model in the invariant, random feature space, on a training set sampled from the reduced core set $\XX_{0}$, has a low expected risk, and generalizes to unseen test points generated from the distribution on $\XX=\XX_{0}\cup \mathcal{G}\XX_0$. The architecture of the proof follows ideas from \cite{RahimiR08} and \cite{Rah_Rec:2008:allerton}.
Recall that given an $L$-Lipschitz convex loss function $V$, our aim is to minimize the expected risk given in Equation \eqref{eq:risk}.
Denote the CDF by $\psi(x,t,\tau)= \mathbb{P}(\scalT{gt}{x}\leq \tau)$, and the empirical CDF by $\hat{\psi}(x,t,\tau)= \frac{1}{|G|}\sum_{i=1}^{|G|} \ind_{\scalT{g_it}{x}\leq \tau}$.
Let $p(t)$ be the distribution of templates $t$. The RKHS defined by the invariant kernel $K_s$, 
$K_{s}(x,z)=\int \int_{-s}^s \psi(x,t,\tau)\psi(z,t,\tau)p(t)dtd\tau$
 denoted $\mathcal{H}_{K_{s}}$ , is the completion of the set of all finite linear combinations of the form:
 \begin{equation}\label{eq:rkhs}
f(x)=\sum_{i} \alpha_iK_{s}(x,x_i),x_i \in \XX, \alpha_i \in \mathbb{R}.
\end{equation}
 Similarly to \cite{Rah_Rec:2008:allerton}, we define the following infinite-dimensional function space:
$$\mathcal{F}_{p}=\left\{ f(x)=\int \int_{-s}^s w(t,\tau)\psi(x,t,\tau)dt d\tau~ \mid \sup_{\tau,t} \frac{|w(t,\tau)|}{p(t)}\leq C  \right\}.$$
\begin{lemma}\label{lem:dense}
$\mathcal{F}_{p}$ is dense in $\mathcal{H}_{K_{s}}$. For $f \in \mathcal{F}_{p}$ we have
$\mathcal{E}_{V}(f)=\int_{\XX_0}\sum_{y\in \YY}V(yf(x)) \rho_{y}(x) d\rho_{\XX}(x),$ where $\XX_{0}$ is the reduced core set.
\end{lemma}

\noindent Since $\mathcal{F}_{p}$ is dense in $\mathcal{H}_{K_{s}}$, we can learn an invariant decision function in the space $\mathcal{F}_{p}$, instead of learning in $\mathcal{H}_{K_{s}}$.
\noindent Let ${\Psi}(x)=\left[\hat{\psi}\left(x,t_j,\frac{sk}{n}\right)\right]_{j=1\dots m,k=-n\dots n }.$
$\Psi$, and $\Phi$ are equivalent up to constants.
\noindent We will approximate the set $\mathcal{F}_{p}$ as follows:
$$\tilde{\mathcal{F}}=\left\{ f(x)=\scalT{w}{\Psi(x)}=\frac{s}{n}\sum_{j=1}^m  \sum_{k=-n}^n w_{j,k}\hat{\psi}\left(x,t_j,\frac{sk}{n}\right), t_j \sim p, j=1\dots m  ~ \mid \nor{w}_{\infty}\leq \frac{C}{m} \right\}. $$
Hence, we learn the invariant decision function via empirical risk minimization where we restrict the function to belong to $\tilde{\mathcal{F}}$, and the sampling in the training set is restricted to the core set $\XX_{0}$.
Note that with this function space we are regularizing for convenience the norm infinity of the weights but this can be relaxed in practice to a classical Tikhonov regularization. 
\begin{theorem}[Learning with Group invariant features]\label{theo:LearningInv}
Let $S=\{(x_i,y_i)\mid x_i \in \XX_{0}, y_i\in \YY,i=1\dots N \}$, a training set sampled from the core set $\XX_{0}$. Let
$f^*_{N}=\argmin_{f\in \mathcal{\tilde{F}}}\hat{\mathcal{E}}_{V}(f)=\frac{1}{N}\sum_{i=1}^N V(y_if(x_i)).\!$ Fix $\delta >0$, then
\begin{align*}
\mathcal{E}_{V}(f^*_{N})&\leq \min_{f\in \mathcal{F}_{p}}\mathcal{E}_{V}(f)+  2 \frac{1}{\sqrt{N}}\left(4LsC+2V(0)+LC\sqrt{\frac{1}{2}\log \left(\frac{1}{\delta}\right)}\right)\\
&+\frac{2sLC}{\sqrt{m}}\left(1+\sqrt{2\log\left(\frac{1}{\delta}\right)}\right)+L\left(\frac{2sC}{\sqrt{|G|}}\left( 1+\sqrt{2\log\left(\frac{m}{\delta} \right)}\right)+\frac{2sC}{n} \right),
\end{align*}
with probability at least $1-3\delta$ on the training set and the choice of templates and group elements.
\end{theorem}
The proof of Theorem \ref{theo:LearningInv} is given in Appendix  B. 
Theorem \ref{theo:LearningInv} shows that learning a linear model in the invariant random feature space defined by $\Phi$ (or equivalently $\Psi$), has a low expected risk. More importantly, this risk is arbitrarily close to the optimal risk achieved in an infinite-dimensional class of functions, namely $\mathcal{F}_{p}$. The  training set is sampled from the reduced core set $\XX_{0}$, and invariant learning generalizes to unseen test points generated from the distribution on $\XX=\XX_{0}\cup \mathcal{G}\XX_0$, hence the reduction in the sample complexity. Recall that $\mathcal{F}_{p}$ is dense in the RKHS of the Haar-integration invariant Kernel, and so the expected risk achieved by a linear model in the invariant random feature space is not far from the one attainable in the invariant RKHS. Note that the error decomposes into two terms.  The first, $O(\frac{1}{\sqrt{N}})$, is statistical and it depends on the training sample complexity $N$. The other is governed by the approximation error of  functions $\mathcal{F}_{p}$, with functions in $\tilde{\mathcal{F}}$, and depends on the number of templates $m$, number of  group elements sampled $|G|$, the number of bins $n$, and has the following form $O(\frac{1}{\sqrt{m}})+O\left(\sqrt{\frac{\log m}{|G|}}\right)+\frac{1}{n}$.\\
\section{Relation to Previous Work}\label{sec:prevwork}
We now put our contributions in perspective by outlining some of the previous work on invariant kernels and approximating kernels with random features. \\
\textbf{Approximating Kernels.} Several  schemes have been proposed for approximating a non-linear kernel with an explicit non-linear feature map in conjunction with linear methods, such as the Nystr\"{o}m method \cite{Williams01usingthe} or random sampling techniques in the Fourier domain for translation-invariant kernels \cite{RahimiR08}. Our features fall under the random sampling techniques where, unlike previous work, we sample both projections and group elements to induce invariance with an integral representation. 
We note that the relation between random features and quadrature rules has been thoroughly studied in \cite{Bach15}, where sharper bounds and error rates are derived, and can apply to our setting.\\
\textbf{Invariant Kernels.} We focused in this paper on Haar-integration kernels {\cite{HaasdonkVB05}}, since they have an integral representation and hence can be represented with random features \cite{Bach15}. Other invariant kernels have been proposed: In \cite{WalderC07} authors introduce transformation invariant kernels, but unlike our general setting, the analysis is concerned with dilation invariance. In \cite{ChoS09}, multilayer arccosine kernels are built by composing kernels that have an integral representation, but does not explicitly induce invariance. More closely related to our work is \cite{BoRF10}, where kernel descriptors are built for visual recognition by  introducing a kernel view of histogram of gradients that corresponds in our case to the cumulative distribution on the group variable. Explicit feature maps are obtained via kernel PCA, while our features are obtained via random sampling. Finally the convolutional kernel network of \cite{MairalKHS14} builds a sequence of multilayer  kernels that have an integral representation, by convolution, considering spatial neighborhoods in an image. Our future work will consider the composition of Haar-integration kernels, where the convolution is applied not only to the spatial variable but to the group variable akin to \cite{scat}.

\section{Numerical Evaluation}

In this paper, and specifically in Theorems 2 and 3, we showed that the random, group-invariant feature map $\Phi$ captures the invariant distance between points, and that learning a linear model trained in the invariant, random feature space will generalize well to unseen test points. In this section, we validate these claims through three experiments. For the claims of Theorem 2, we will use a nearest neighbor classifier, while for Theorem 3, we will rely on the regularized least squares (RLS) classifier, one of the simplest algorithms for supervised learning. While our proofs focus on norm-infinity regularization, RLS corresponds to Tikhonov regularization with square loss. Specifically, for performing $T-$way classification on a batch of $N$ training points in $\mathbb R^d$, summarized in the data matrix $X\in \mathbb R^{N\times d}$ and label matrix $Y\in \mathbb R^{N\times T}$, RLS will perform the optimization, $\min_{W\in \mathbb R^{m\times T}}\left\{\frac{1}{N}||Y - \Phi(X)W||_F^2 + \lambda ||W||_F^2\right\}$, where $||\cdot||_F$ is the Frobenius norm, $\lambda$ is the regularization parameter, and $\Phi$ is the feature map, which for the representation described in this paper will be a CDF pooling of the data projected onto group-transformed random templates. All RLS experiments in this paper were completed with the GURLS toolbox \cite{gurls}. The three datasets we explore are:\\
$\mathbf{X_{perm}}$ (Figure \ref{fig:fig1}): An artificial dataset consisting of all sequences of length 5 whose elements come from an alphabet of 8 characters. We want to learn a function which assigns a positive value to any sequence that contains a target set of characters (in our case, two of them) regardless of their position. Thus, the function label is globally invariant to permutation, and so we project our data onto all permuted versions of our random template sequences.\\
\textbf{MNIST} (Figure \ref{fig:fig2}): We seek local invariance to translation and rotation, and so all random templates are translated by up to 3 pixels in all directions and rotated between -20 and 20 degrees. \\
\textbf{TIDIGITS} (Figure \ref{fig:fig3}): We use a subset of TIDIGITS consisting of 326 speakers (men, women, children) reading the digits 0-9 in isolation, and so each datapoint is a waveform of a single word. We seek local invariance to pitch and speaking rate \cite{speechvariability}, and so all random templates are pitch shifted up and down by 400 cents and warped to play at half and double speed. The task is 10-way classification with one class-per-digit. See \cite{voinea} for more detail.\\

\begin{figure*}
\centerline{\epsfig{figure= 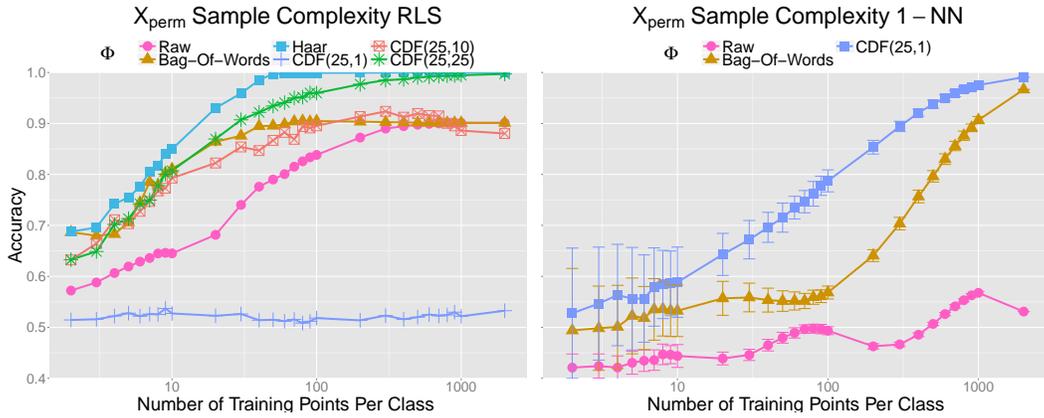, width=140mm, height=55mm}}
\caption{{\it Classification accuracy as a function of training set size, averaged over 100 random training samples at each size. $\Phi = \text{CDF}(n,m)$ refers to a random feature map with $n$ bins and $m$ templates. With 25 templates, the random feature map outperforms the raw features and a bag-of-words representation (also invariant to permutation) and even approaches an RLS classifier with a Haar-integration kernel. Error bars were removed from the RLS plot for clarity. See supplement. }}
\label{fig:fig1}
\end{figure*}

\begin{figure*}
\centerline{\epsfig{figure= 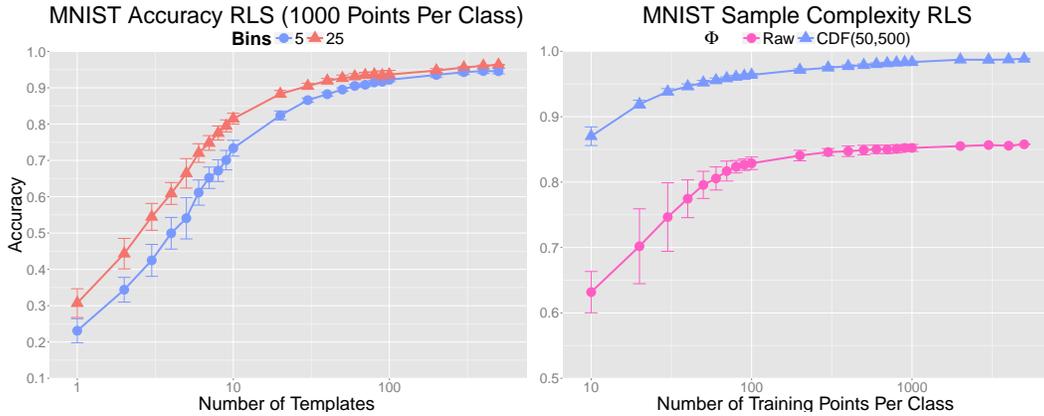, width=140mm, height=55mm}}
\caption{{\it Left Plot) Mean classification accuracy as a function of number of bins and templates, averaged over 30 random sets of templates. Right Plot) Classification accuracy as a function of training set size, averaged over 100 random samples of the training set at each size. At 1000 examples per class, we achieve an accuracy of 98.97\%.}}
\label{fig:fig2}
\end{figure*}

\begin{figure*}
\centerline{\epsfig{figure= 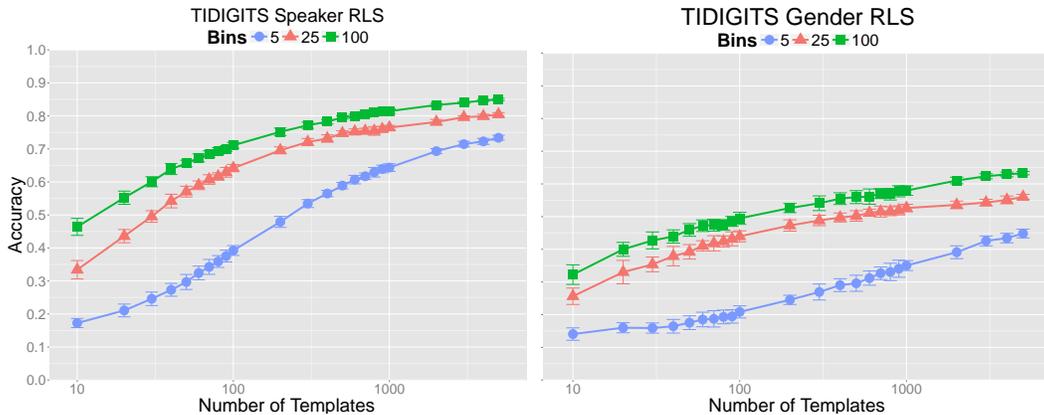, width=140mm, height=55mm}}
\caption{{\it Mean classification accuracy as a function of number of bins and templates, averaged over 30 random sets of templates. In the ``Speaker'' dataset, we test on unseen speakers, and in the ``Gender'' dataset, we test on a new gender, giving us an extreme train/test mismatch. \cite{speechvariability}.}}
\label{fig:fig3}
\end{figure*}
\textbf{Acknowledgements:} Stephen Voinea acknowledges the support of a Nuance Foundation Grant. This work was also supported in part by the Center for Brains, Minds and Machines (CBMM), funded by NSF STC award  CCF Ð 1231216.

\appendix
\section{Proofs of Theorems \ref{theo:CDF_RP} and \ref{theo:InvJL}}\label{sec:proofs}
\begin{proof}[Proof of Theorem \ref{theo:CDF_RP}]
1)
\begin{eqnarray*}
K_{s}(x,z)&=& \mathbb{E}_{t} \int_{-s}^s  \mathbb{E}_{g} \left[\ind_{\scalT{x}{gt}\leq \tau}\right] \mathbb{E}_{g'}\left[ \ind_{\scalT{z}{g't}\leq \tau}\right]  d\tau\\
&=& \mathbb{E}_{t} \int d\mu(g) d\mu(g') \int_{-s}^s \ind_{\scalT{x}{gt}\leq \tau}\ind_{\scalT{x}{g't}\leq \tau}d\tau\\
&= & \int d\mu(g) d\mu(g')  \mathbb{E}_{t} \left(s- \max (\scalT{x}{gt}, \scalT{z}{g't})\right).
\end{eqnarray*}
where the second equality is by Fubini theorem and the last one holds since for $a, b \in [-s,s]$ :
$$\int_{-s}^s \ind_{a\leq \tau}\ind_{b\leq \tau} d\tau= s-\max(a,b).$$
Recall that the sampling of $t$ is the following for  $\eps \in (0,1)$ let : 
$$t = n \sim \mathcal{N}\left(0, \frac{1}{d} I_d\right) , \text{ if } \nor{n}^2_2 < 1+\eps, t=\perp \text{ else },$$
since our group is unitary, $x$ being norm one, and by virtue of this sampling the dot product $\left|\scalT{x}{gt}\right| \leq \nor{n}_{2}\leq \sqrt{1+\eps} \leq 1+\eps$ . Hence $\scalT{x}{gt} \in [-(1+\eps),1+\eps]$, and  we can choose $s=1+\eps$.
Using again the fact the group is unitary and compact we have:
$$K_{s}(x,z)=  \int d\mu(g) d\mu(g')  \mathbb{E}_{t} (s- \max \left(\scalT{g^{-1}x}{t}, \scalT{g^{',-1}z}{t}\right).$$
 Now using  this particular sampling of templates we have:
 $$K_{s}(x,z)=  \int_{G}\int_{G} d\mu(g)d\mu(g') \mathbb{E}_{n}\left( \ind_{\nor{n}^2_2 < 1+\eps}\left[1+\eps- \max \left( \scalT{g^{-1}x}{n},\scalT{g'^{-1}z}{n}\right)\right]\right).$$
 Let $$Z_{x,z} (n,g,g')=  \max \left( \scalT{g^{-1}x}{n},\scalT{g'^{-1}z}{n}\right),$$
It follows that:
\begin{align}\label{eq:K}
K_{s}(x,z)&=  \int_{G}\int_{G} d\mu(g)d\mu(g') \mathbb{E}_{n}\left( \ind_{\nor{n}^2_2 < 1+\eps}\left[1+\eps -Z_{x,z}(n,g,g') \right]\right) \nonumber\\
&=(1+\eps) \mathbb{P}(\nor{n}^2_2 < 1+\eps) - \int_{G}\int_{G} d\mu(g)d\mu(g') \mathbb{E}_{n}\left( \ind_{\nor{n}^2_2 < 1+\eps} Z_{x,z}(n,g,g')\right)\nonumber\\
&= (1+\eps)\mathbb{P}(\nor{n}^2_2 < 1+\eps) - \int_{G}\int_{G} d\mu(g)d\mu(g') \mathbb{E}_{n} \left((1 - \ind_{\nor{n}^2_2\geq 1+\eps})Z_{x,z}(n,g,g')  \right)\nonumber\\
&= (1+\eps)\mathbb{P}(\nor{n}^2_2 < 1+\eps) - \int_{G}\int_{G} d\mu(g)d\mu(g') \mathbb{E}_{n}Z_{x,z}(n,g,g')\nonumber \\
&+ \int_{G}\int_{G} d\mu(g)d\mu(g') \mathbb{E}_{n}\left( \ind_{\nor{n}^2_2\geq1+\eps}Z_{x,z}(n,g,g')\right)
\end{align}

We are left with evaluating or bounding two expectations: 
$ I_{1}= \mathbb{E}_{n}Z_{x,z}(n,g,g')$, and $I_{2}=\mathbb{E}_{n}\left( \ind_{\nor{n}^2_2 \geq 1+\eps}Z_{x,z}(n,g,g')\right),$
that involve the maximum of correlated gaussian variables as we will see in the following.
 
By rotation invariance of Gaussians we have that $\scalT{g^{-1}x}{n}$, and $\scalT{g'^{-1}z}{n}$ are two correlated random gaussian variables with correllation coefficient that we note by $\cos(\theta_{g,g'})= \scalT{g^{-1}x}{g^{,-1}z}$. Hence by a change of a basis we can write: 
$$ \scalT{g^{-1}x}{n}= \frac{1}{\sqrt{d}}u,~  \scalT{g'^{-1}z}{n}=\frac{1}{\sqrt{d}}\cos (\theta_{g,g'})u +\frac{1}{\sqrt{d}}\sqrt{1-\cos^2(\theta_{g,g'}) }v $$
where $\cos(\theta_{g,g'})= \scalT{g^{-1}x}{g'^{-1}z}$, and $u , v \sim \mathcal{N}(0,1)$ iids.

Hence, $$I_{1}=\frac{1}{\sqrt{d}} \mathbb{E}_{u,v} \max\left(u, \cos (\theta_{g,g'})u+\sqrt{1-\cos^2(\theta_{g,g'}) }v  \right).$$

The following Lemma  from  \cite{Clark1961} gives the expectation and the variance of the maximum of two gaussians with correllation coefficient $\rho$.
\begin{lemma} [Mean and Variance of Maximum of Correlated Gaussians \cite{Clark1961} ]\label{lem:max}
Let $X \sim \mathcal{N}(\mu_{X}, \sigma^2_{X})$ and $Y \sim \mathcal{N}(\mu_{Y}, \sigma^2_{Y})$, two correlated gaussians with correllation coefficient $\rho$.
Define $\phi_{\mathcal{N}}(x)= \frac{1}{\sqrt{2\pi}}\exp(-x^2/2) $, and $\Phi_{\mathcal{N}}(y)=\int_{-\infty}^y \phi_{\mathcal{N}}(x) dx $. 
Let $a= \sqrt{\sigma^2_{X}+\sigma^2_{Y}- 2\rho \sigma_{X}\sigma_{Y}}$, and $\alpha =\frac{\mu_{X}-\mu_{Y}}{a}$.\\
The mean $\mu_{Z}$  and variance $\sigma^2_{Z}$ of $Z=\max(X,Y)$ are expressed analytically as follows: 
\begin{align}
\mu_{Z}&= \mu_{X} \Phi_{\mathcal{N}}(\alpha) + \mu_{Y} \Phi_{\mathcal{N}}(-\alpha)+ a \phi_{\mathcal{N}}(\alpha).  \\
\sigma^2_{Z}&= \underbrace{\left(\sigma^2_{X}+ \mu^2_{X}\right) \Phi_{\mathcal{N}}(\alpha) + \left(\sigma^2_{Y}+ \mu^2_{Y}\right)\Phi_{\mathcal{N}}(-\alpha)+\left( \mu_{X}+\mu_{Y}\right)  a \phi_{\mathcal{N}}(\alpha)}_{\mathbb{E}Z^2}- \mu^2_{Z}.
\end{align}
\end{lemma}

Applying Lemma \ref{lem:max} to our case $(\mu_{X}=\mu_{Y}=0, \sigma_{X}=\sigma_{Y}=1, \rho = \cos(\theta_{g,g'}))$. We have:
$a= \sqrt{2 (1-\cos(\theta_{g,g'}))}$ and $\alpha=0$.
\begin{eqnarray}\label{eq:I1}
I_{1} &=&\frac{1}{\sqrt{d}} a \phi_{\mathcal{N}}(0) \nonumber \\
&=& \frac{1}{\sqrt{2\pi d}}\sqrt{2 (1-\cos(\theta_{g,g'}))}  \nonumber \\
&=&  \frac{1}{\sqrt{2\pi d}} \nor{g^{-1}x-g'^{-1}z}_2 .
\end{eqnarray}
We turn now to $I_2$ that we bound using Cauchy-Schwarz inequality:
\begin{eqnarray} \label{eq:I2}
\left| I_{2}\right|&=&\left|\mathbb{E}_{n}\left( \ind_{\nor{n}^2_2 \geq 1+\eps}Z_{x,z}(n,g,g')\right) \right|  \nonumber\\
&\leq& \sqrt{E (\ind_{\nor{n}^2_2 \geq 1+\eps})} \sqrt{E(Z^2_{x,z}(n,g,g'))} \nonumber \\
&=& \sqrt{ \mathbb{P}\left(\nor{n}^2_2\geq 1+\eps \right)}\sqrt{E(Z^2_{x,z}(n,g,g'))}.
\end{eqnarray}
%
On the first hand, applying again Lemma \ref{lem:max} (for $\mathbb{E} Z^2$) we have:
\begin{eqnarray}\label{eq:Second}
E(Z^2_{x,z}(n,g,g') &= &\frac{1}{d}\mathbb{E}_{u,v}\left( \max\left(u, \cos (\theta_{g,g'})u+\sqrt{1-\cos^2(\theta_{g,g'}) }v  \right)\right)^2 \nonumber\\
&=& \frac{1}{d} \left(2 \Phi_{\mathcal{N}}(0)\right) \nonumber\\
&=& \frac{1}{d}.
\end{eqnarray}
On the other hand, note that $\nor{n}^2_{2}$ has  a (normalized) chi squared distribution with $d$ degree of freedom $\chi^2_{d}$ , with mean $1$ .
 The following Lemma gives upper  bounds for  the upper and lower tails of a chi square distribution.
\begin{lemma}[$\chi^2$ tail bounds]\label{lem:tails}
Let $X\sim \chi^2_{k}$, a chi squared random variable with $k$ degree of freedom. The following  hold true  for any $\eps \in (0,1)$:
\begin{itemize}
\item  Upper Bound for the upper tail \cite{VershyninReview}: $\mathbb{P}\left( \frac{1}{k} X \geq 1+\eps \right) \leq e^{-k \eps^2/8} $.
\item  Upper Bound for the lower tail \cite{Chisquare}:  For all $k \geq 2$, $u\geq k-1$ we have:
$$\mathbb{P }\left(X < u\right) \leq 1-\frac{1}{2} \exp\left( -\frac{1}{2} \left( u-k - (k-2)\log(u/k)+\log(k)\right)\right).$$
More specifically for $u=k(1+\eps)$ we have:
$$\mathbb{P }\left(\frac{1}{k} X < 1+\eps\right) \leq 1-\frac{1}{2}\frac{e^{-\eps k/2} \left(1+\eps\right)^{\frac{k-2}{2}}}{\sqrt{k}}.$$

\end{itemize}
\end{lemma}
Applying Lemma \ref{lem:tails}, for $\nor{n}^2_{2}$. We have $\nor{n}^2_{2} =\frac{1}{d} X$, where $X \sim \chi^2_{d}$, hence:
\begin{equation} \label{eq:normG}
\mathbb{P}\left(\nor{n}^2_2 \geq 1+\eps \right) \leq e^{- d \eps^2/8} ,
\end{equation}

Putting together Equations \eqref{eq:I2},\eqref{eq:normG}, \eqref{eq:Second} we have finally: 
\begin{equation}\label{eq:absI2}
\left| I_2 \right| \leq \frac{e^{- d \eps^2/16}}{\sqrt{d}}.
\end{equation}
Putting together Equations \eqref{eq:K}, \eqref{eq:I1}, and \eqref{eq:absI2}, and using upper and lower bounds for $\mathbb{P}(\nor{n}^2_2 < 1+\eps)$ from Lemma \ref{lem:tails}:
\begin{eqnarray*}
 K_{s}(x,z) &\leq& \left(1+\eps \right)\mathbb{P}(\nor{n}^2_2 < 1+\eps) - \frac{1}{\sqrt{2\pi d}}\int_{G}\int_{G}\nor{g^{-1}x-g'^{-1}z}_2  d\mu(g)d\mu(g') +  \frac{e^{- d \eps^2/16}}{\sqrt{d}}\\
 & \leq& \left(1+\eps \right)\left(1-\frac{1}{2}\frac{e^{-\eps d/2} \left(1+\eps\right)^{\frac{d-2}{2}}}{\sqrt{d}}\right) - \frac{1}{\sqrt{2\pi d}}\int_{G}\int_{G} \nor{g^{-1}x-g'^{-1}z}_2 d\mu(g)d\mu(g') \\
 &+&  \frac{e^{- d \eps^2/16}}{\sqrt{d}}.  \\
 K_{s}(x,z) & \geq&(1+\eps)  \mathbb{P}(\nor{n}^2_2 < 1+\eps) - \frac{1}{\sqrt{2\pi d}}\int_{G}\int_{G}  \nor{g^{-1}x-g'^{-1}z}_2d\mu(g)d\mu(g') - \frac{e^{- d \eps^2/16}}{\sqrt{d}}\\
 &\geq& (1+\eps)\left( 1- e^{-d\eps^2/8}\right)- \frac{1}{\sqrt{2\pi d}}\int_{G}\int_{G}  \nor{g^{-1}x-g'^{-1}z}_2d\mu(g)d\mu(g') - \frac{e^{- d \eps^2/16}}{\sqrt{d}}.
\end{eqnarray*}
Noting by  $d_{G}$ the integral and using that the group is compact and unitary:
\begin{eqnarray*}
d_{G}(x,z)&=& \frac{1}{\sqrt{2\pi d}}\int_{G}\int_{G}\nor{g^{-1}x-g'^{-1}z}_2  d\mu(g)d\mu(g') \\
&=& \frac{1}{\sqrt{2\pi d}}\int_{G}\int_{G}  \nor{gx-g'z}_2 d\mu(g)d\mu(g').
\end{eqnarray*}
We finally have:
\begin{equation}
 -  \frac{e^{- d \eps^2/16}}{\sqrt{d}}- (1+\eps)e^{-d\eps^2/8}+\eps\leq K_{s}(x,z) - \left( 1-d_{G}(x,z)\right)\leq   \frac{e^{- d \eps^2/16}}{\sqrt{d}} -\frac{1}{2}\frac{e^{-\eps d/2} \left(1+\eps\right)^{\frac{d}{2}}}{\sqrt{d}}+\eps .
\end{equation}
For any $\eps \in (0,1)$ , as the dimension $d \to \infty$, we have asymptotically: 
$$K_{s}(x,z) \to 1- d_{G}(x,z)+\eps = s-d_{G}(x,z).$$

2) The symmetry of $K$ is obvious. Let $p(t)$ be the distribution of the templates $t$. 
Define the following weighted dot product: $\scalT{f(x,.,.)}{g(z,.,.)}= \int_{t} p(t) \int_{-s}^s d\tau f(x,t,\tau)g(z,t,\tau) $.
Recall that:
\begin{eqnarray*}
K_{s}(x,z)&=& \int p(t) dt  \int_{-s}^s  \psi(x,t,\tau )\psi(z,t,\tau)  d\tau \\
&=& \scalT{\psi (x,.,.) }{\psi(z,.,.)}. 
\end{eqnarray*}
Hence $K$ is symmetric and positive semidefinite.


\end{proof}

\begin{proof}  [Proof of Theorem \ref{theo:InvJL}]
In the following we fix two points  $x$ and $z$ in $\mathcal{X}$ and a random template $t$.
Let $X_j=\int_{-s}^s \mathbb{P}(\scalT{gt_j}{x}\leq \tau)\mathbb{P}(\scalT{gt_j}{z}\leq \tau)d\tau$, we have $0\leq X_j\leq 2s$, where $s=1+\eps$. 
Recall that $K_s(x,z)=\frac{1}{m} \mathbb{E}_{t}(\sum_{j=1}^m X_j)$.
By Hoeffding's inequality we have:
$$\mathbb{P}_{t}\left\{\left|\frac{1}{m} \sum_{j=1}^m X_j-K_s(x,z)\right|> \epsilon\right\} \leq 2 \exp\left(\frac{-2m\epsilon^2}{(2s)^2}\right)$$
 Turning now to the CDF $\psi(x,t,\tau)= \mathbb{P}(\scalT{gt}{x}\leq \tau)$, and the empirical CDF $\hat{\psi}(x,t,\tau)= \frac{1}{|G|}\sum_{i=1}^{|G|} \ind_{\scalT{g_it}{x}\leq \tau}$.
 By the theorem on convergence of the empirical CDF \cite{dvoretzky1956} (Theorem \ref{theo:empCDF} given in Appendix \ref{app:TT} ) we have, for $\gamma>0$:
$$\mathbb{P}_{g}\left\{\sup _{\tau }\left| \hat{\psi}(x,t,\tau)-\psi(x,t,\tau) \right|> \gamma \right\} \leq 2\exp (-2|G|\gamma^2) $$
Hence we have $\forall \tau \in [-s,s]$:
$$\left|\hat{\psi}(x,t,\tau)- \psi(x,t,\tau)\right|\leq  \gamma \text{ and } \left|\hat{\psi}(x,t,\tau)-\psi(z,t,\tau)\right|\leq  \gamma $$
with a probability at least $1-4\exp(-2|G|\gamma^2)$.\\
Define $X= \int_{-s}^s \psi(x,t,\tau)\psi(z,t,\tau) d\tau$,  $\hat{X}=\int_{-s}^s\hat{\psi}(x,t,\tau)\hat{\psi}(z,t,\tau)d\tau$, and $\tilde{X}=\frac{(2s)}{n} \sum_{k=-n}^n\hat{\psi}(x,t,\frac{ks}{n})\hat{\psi}(z,t,\frac{ks}{n}) $, choose $0<\gamma<1$:
\begin{eqnarray*}
|\hat{X}-X|&=&\left|\int_{-s}^s \left(\hat{\psi}(x,t,\tau)\hat{\psi}(z,t,\tau)-{\psi}(x,t,\tau){\psi}(z,t,\tau) \right) d\tau\right|\\
&=&\left|\!\int_{-s}^s\!\left(\hat{\psi}(x,t,\tau)-\psi(x,t,\tau)+ \psi(x,t,\tau)\right)\left(\hat{\psi}(z,t,\tau)-\psi(z,t,\tau)+ \psi(z,t,\tau)\right)-\psi(x,t,\tau)\psi(z,t,\tau)d\tau\right|\\
&\leq& (2\gamma+\gamma^2)2s\\
&\leq& 6s\gamma,
\end{eqnarray*}
with probability $1-4\exp(-2|G|\gamma^2)$.
Define $X_j= \int_{-s}^s \psi(x,t_j,\tau)\psi(z,t_j,\tau) d\tau$,  $\hat{X}_j=\int_{-s}^s\hat{\psi}(x,t_j,\tau)\hat{\psi}(z,t_j,\tau)d\tau$, and $\tilde{X}_j=\frac{(2s)}{n} \sum_{k=-n}^n\hat{\psi}(x,t_j,\frac{ks}{n})\hat{\psi}(z,t_j,\frac{ks}{n}) $,
Then for all $j=1\dots m $, we have $$|\hat{X}_j-X_j|\leq  6s\gamma$$ with probability $1-4m\exp(-2|G|\gamma^2)- 2 \exp\left(\frac{-2m\epsilon^2}{(2s)^2}\right).$

\noindent Now we turn to the numerical approximation of the integra by a Riemann sum, we have for all $j=1\dots m$ :
$$\left| \hat{X}_j-\tilde{X_j} \right|\leq \frac{s}{n}.$$
Hence the error decomposes in the following way:
\begin{align*}
\left|\scalT{\Phi(x)}{\Phi(z)}-K_s(x,z)\right|&=\left|\frac{1}{m}\sum_{j=1}^m \tilde{X}_j - K_s(x,z)\right|\\
&=\left|\left(\frac{1}{m}\sum_{j=1}^m \tilde{X}_j - \frac{1}{m}\sum_{j=1}^m \hat{X}_j\right) +\left( \frac{1}{m}\sum_{j=1}^m \hat{X}_j - \frac{1}{m}\sum_{j=1}^m {X}_j \right)+\left(\frac{1}{m}\sum_{j=1}^m {X}_j - K_s(x,z)\right)  \right| \\
&\leq \underbrace{\left|\frac{1}{m}\sum_{j=1}^m \tilde{X}_j - \frac{1}{m}\sum_{j=1}^m \hat{X}_j\right|}_{\text{Numerical Binning Error}}+\underbrace{\left| \frac{1}{m}\sum_{j=1}^m \hat{X}_j - \frac{1}{m}\sum_{j=1}^m {X}_j\right|}_{\text{Group CDF Approximation Error }}+\underbrace{\left| \frac{1}{m}\sum_{j=1}^m {X}_j - K_s(x,z)\right|}_{\text{Templates Concentration Error}}\\
&\leq \frac{s}{n}+ 6s\gamma+\epsilon.
\end{align*}
with probability $1-4m\exp(-2|G|\gamma^2)- 2 \exp\left(\frac{-2m\epsilon^2}{(2s)^2}\right).$
For this to hold on all pairs of points in a set of cardinality $N$ we have:
$$\left|\scalT{\Phi(x_i)}{\Phi(x_j)}-K(x_i,x_j)\right| \leq  \frac{s}{n}+ 6s\gamma+\epsilon, i=1\dots N, j=1\dots N ,$$
with probability $1-4m N(N-1)\exp(-2|G|\gamma^2)- 2N(N-1) \exp\left(\frac{-m\epsilon^2}{2(s)^2}\right).$\\
Hence we have for numerical constants $C_1$, and $C_2$, $0<\delta_1,\delta_2<1$, and $0<\eps_0,\eps_1,\eps_2<1$, for $n\geq \frac{s}{\eps_0}$, $m\geq \frac{C_1 }{\eps^2_1}\log(\frac{N}{\delta_1})$,$|G|\geq\frac{C_2}{\eps^2_2}\log(\frac{Nm}{\delta_2})$,
:
$$\left|\scalT{\Phi(x_i)}{\Phi(x_j)}-K_{s}(x_i,x_j)\right| \leq  \eps_{0}+\eps_{1}+\eps_{2} , i=1\dots N, j=1\dots N ,$$
with probability $1-\delta_1-\delta_2$.

\end{proof}

\section{Proof of Theorem \ref{theo:LearningInv}}\label{sec:theo3}
\begin{proof}[Proof of Lemma \ref{lem:dense}]
Our proof parallels similar proofs in \cite{Rah_Rec:2008:allerton}. Note that functions of the form \eqref{eq:rkhs} are dense in $\mathcal{H}_{K}$.
$f(x)=\sum_{i}\alpha_iK_s(x,x_i)= \sum_{i}\alpha_i \int\int_{-s}^s\psi(x,t,\tau)\psi(x_i,t,\tau)p(t) dt d\tau$\\
$=\int\int_{-s}^s\left(p(t)\sum_{i}\alpha_i \psi(x_i,t,\tau)\right)\psi(x,t,\tau)dt d\tau.$
Let $\beta(t,\tau)=p(t)\sum_{i}\alpha_i \psi(x_i,t,\tau)$, since $0\leq \psi(x,t,\tau)\leq 1$,
$\forall x,t,\tau$, we have $\frac{|\beta(t,\tau)|}{p(t)}\leq \sum_{i}|\alpha_i|<\infty$, since $\alpha_i$ are finite. Hence $f$ can be written in the form: 
$$f(x)=\int \int_{-s}^s \beta(t,\tau)\psi(x,t,\tau)dtd\tau,~ \sup_{\tau,t}\frac{|\beta(t,\tau)|}{p(t)}<\infty,$$
and $f\in \mathcal{F}_{p}.$
\end{proof}

\noindent In order to prove Theorem \ref{theo:LearningInv}, we need some preliminary lemmas.
\noindent The following Lemma assess the approximation  of any function $f \in \mathcal{F}_{p}$, by a certain  $\tilde{f} \in \tilde{\mathcal{F}}$.\\  

\begin{lemma}[$\tilde{\mathcal{F}}$ Approximation of $\mathcal{F}_{p}$]\label{lem:approx}
Let $f$ be a function in $\mathcal{F}_{p}$. Then for $\delta_1,\delta_2>0$, there exists a function $\tilde{f} \in \tilde {\mathcal{F}}$ such that:
$$\nor{\tilde{f} -f}_{\mathcal{L}^2(\XX,\rho_{\XX})}\leq  \frac{2sC}{\sqrt{m}}\left(1+\sqrt{2\log\left(\frac{1}{\delta_1}\right)}\right)+\frac{2sC}{\sqrt{|G|}}\left( 1+\sqrt{2\log\left(\frac{m}{\delta_2} \right)}\right)+\frac{2sC}{n} ,$$
with probability at least $1-\delta_1-\delta_2$.
\end{lemma}
\begin{proof}[Proof of Lemma \ref{lem:approx}] Let $f \in \mathcal{F}_{p}, f(x)=\int \int_{-s}^s w(t,\tau)\psi(x,t,\tau) d\tau dt $.\\ $\text{Let } {f_j}(x)=\int_{-s}^s\frac{w(t_j,\tau)}{p(t_j)}\psi(x,t_j,\tau)d\tau, ~\hat{f_j}(x)=\int_{-s}^s\frac{w(t_j,\tau)}{p(t_j)}\hat{\psi}(x,t_j,\tau)d\tau, \text{ and }$
$\tilde{f_j}(x)=\frac{s}{n}\sum_{k=-n}^n\frac{w(t_j,\frac{ks}{n})}{p(t_j)}\hat{\psi}(x,t_j,\frac{ks}{n}).$
We have the following: $\mathbb{E}_{t}(f_j)=f$, and $\frac{1}{m}\mathbb{E}_{t}(\sum_{j=1}^m f_j)=f$.
Consider the Hilbert space $\mathcal{L}^2(\XX,\rho_{\XX})$, with dot product:
$\scalT{f}{g}_{\mathcal{L}^2(\XX,\rho_{\XX})}=\int_{\XX} f(x)g(x)d\rho_{\XX}(x)$.\\
Note that : $\int_{-s}^s g(\tau)d\tau \leq \sqrt{2s}\sqrt{\int_{-s}^s g^2(\tau)d\tau}$ 
$$|| f_j||_{\mathcal{L}^2(\XX,\rho_{\XX})} = \sqrt{\int_{\XX} \left(\int_{-s}^s\frac{w(t_j,\tau)}{p(t_j)}\psi(x,t_j,\tau)d\tau \right)^2d\rho_{\XX}(x)}\leq (2sC),$$
Fix $\delta_1>0$, applying Lemma \ref{lem:ConcH} we have therefore with probability $1-\delta_1$:\\
\begin{equation}\label{eq:1}
\nor{\frac{1}{m}\sum_{j=1}^m f_j -f}_{\mathcal{L}^2(\XX,\rho_{\XX})}\leq \frac{2sC}{\sqrt{m}}\left(1+\sqrt{2\log\left(\frac{1}{\delta_1}\right)}\right),
\end{equation}
Now turn to:
\begin{equation*}
\nor{\frac{1}{m}\sum_{j=1}^m (\hat{f_j} -f_j)}_{\mathcal{L}^2(\XX,\rho_{\XX})}\leq \frac{1}{{m}}\sum_{j=1}^m \nor{\hat{f}_j-f_j}_{\mathcal{L}^2(\XX,\rho_{\XX})}, 
\end{equation*}
\begin{align*}
\nor{\hat{f}_j-f_j}^2_{\mathcal{L}^2(\XX,\rho_{\XX})}&= \int_{\XX} \left(\int_{-s}^s\frac{w(t_j,\tau)}{p(t_j)}(\psi(x,t_j,\tau)-\hat{\psi}(x,t_j,\tau))d\tau \right)^2d\rho_{\XX}(x) \\
& \leq 2s  \int_{\XX} \int_{-s}^s\frac{w^2(t_j,\tau)}{p^2(t_j)}(\psi(x,t_j,\tau)-\hat{\psi}(x,t_j,\tau))^2d\tau d\rho_{\XX}(x)\\
&\leq  2sC^2 \int_{\XX} \int_{-s}^{s}(\hat{\psi}(x,t_j,\tau)-\psi(x,t_j,\tau))^2d\tau d\rho_{\XX}(x)\\
&=  2sC^2 \int_{-s}^s \int_{\XX}(\hat{\psi}(x,t_j,\tau)-\psi(x,t_j,\tau))^2 d\rho_{\XX}(x) d\tau\\
&= 2sC^2  \int_{-s}^s \nor{\hat{\psi}(.,t_j,\tau)-\psi(.,t_j,\tau)}^2_{\mathcal{L}^2(\XX,\rho_{\XX})} d\tau\\
&\leq (2sC)^2 \sup_{\tau,j=1\dots m}\nor{\hat{\psi}(.,t_j,\tau)-\psi(.,t_j,\tau)}^2_{\mathcal{L}^2(\XX,\rho_{\XX})}.
\end{align*}
Recall that: $\hat{\psi}(x,t,\tau)= \frac{1}{|G|}\sum_{i=1}^{|G|} \ind_{\scalT{g_it}{x}\leq \tau}$, and $\psi(x,t,\tau)=\mathbb{E}_{g}\hat{\psi}(x,t,\tau)$.\\
Clearly $\nor{\ind_{\scalT{.}{gt}\leq \tau}}_{\mathcal{L}_{2}(\XX,\rho_{\XX})}\leq 1$, hence applying again Lemma \ref{lem:ConcH}, for $\delta_2>0$ we have with probability $1-\delta_2$:
$$\nor{\hat{\psi}(.,t_j,\tau)-\psi(.,t_j,\tau)  }^2_{\mathcal{L}^2(\XX,\rho_{\XX})}\leq \frac{1}{|G|}\left( 1+\sqrt{2\log\left(\frac{1}{\delta_2} \right)}\right)^2,$$
It follows that: 
$\forall j=1\dots m ,\nor{\hat{f}_j-f_j}\leq \frac{2Cs}{\sqrt{|G|}}\left( 1+\sqrt{2\log\left(\frac{1}{\delta_2} \right)}\right)$, with probability $1-m\delta_2$.
Hence with probability $1-m\delta_2$, we have:
\begin{equation}\label{eq:2}
\nor{\frac{1}{m}\sum_{j=1}^m (\hat{f_j} -f_j)}_{\mathcal{L}^2(\XX,\rho_{\XX})}\leq \frac{2Cs}{\sqrt{|G|}}\left( 1+\sqrt{2\log\left(\frac{1}{\delta_2} \right)}\right).
\end{equation}
and by the approximation of a Riemann sum we have that:
\begin{equation}\label{eq:3}
\nor{\frac{1}{m}\sum_{j=1}^m (\hat{f_j} -\tilde{f_j})}_{\mathcal{L}^2(\XX,\rho_{\XX})}\leq \frac{2sC}{n} .
\end{equation}
It is clear that $\tilde{f}=\frac{1}{m}\sum_{j=1}^m \tilde{f}_j \in \tilde{\mathcal{F}}$, hence, putting together equations \eqref{eq:1},\eqref{eq:2}, and \eqref{eq:3} we finally have:
\begin{align*}
\nor{\frac{1}{m}\sum_{j=1}^m \tilde{f_j} -f}_{\mathcal{L}^2(\XX,\rho_{\XX})}&\leq\nor{\frac{1}{m}\sum_{j=1}^m (\tilde{f_j}-\hat{f}_j)}_{\mathcal{L}^2(\XX,\rho_{\XX})}+\nor{\frac{1}{m}\sum_{j=1}^m (\hat{f_j} -f_j)}_{\mathcal{L}^2(\XX,\rho_{\XX})}+\nor{\frac{1}{m}\sum_{j=1}^m f_j -f}_{\mathcal{L}^2(\XX,\rho_{\XX})}\\
&\leq  \frac{2sC}{n} +\frac{2Cs}{\sqrt{|G|}}\left( 1+\sqrt{2\log\left(\frac{1}{\delta_2} \right)}\right)+ \frac{2sC}{\sqrt{m}}\left(1+\sqrt{2\log\left(\frac{1}{\delta_1}\right)}\right)
\end{align*} 
with probability $1-\delta_1-m\delta_2$.
\end{proof}

The following Lemma shows how the approximation of functions in $\mathcal{F}_{p}$, by functions in $\tilde{\mathcal{F}}$, translates to the expected Risk:
\begin{lemma}[Bound on the Approximation Error]\label{lem:comp}
Let $f\in \mathcal{F}_{p}$, fix $\delta_1,\delta_2>0$. There exists a function $\tilde{f} \in \tilde{\mathcal{F}}$, such that:
$$\mathcal{E}_{V}(\tilde{f})\leq \mathcal{E}_{V}(f)+  \frac{2sLC}{\sqrt{m}}\left(1+\sqrt{2\log\left(\frac{1}{\delta_1}\right)}\right)+L\left(\frac{2sC}{\sqrt{|G|}}\left( 1+\sqrt{2\log\left(\frac{m}{\delta_2} \right)}\right)+\frac{2sC}{n} \right), $$
with probability at least $1-\delta_1-\delta_2$.
\end{lemma}
\begin{proof}[Proof of Lemma \ref{lem:comp}]$\mathcal{E}_{V}(\tilde{f})-\mathcal{E}_{V}(f) \leq \int_{\XX}\left|V(y\tilde{f}(x))-V(yf(x))\right|d\rho_{\XX}(x)\leq L \int_{\XX}|\tilde{f}(x)-f(x)|d\rho_{\XX}(x)\leq L \sqrt{\int_{\XX}(\tilde{f}(x)-f(x))^2d\rho_{\XX}(x)}=L \nor{\tilde{f}-f}_{\mathcal{L}^2(\XX,\rho_{\XX})},$ where we used the Lipschitz condition and Jensen inequality. The rest of the proof follows from Lemma \ref{lem:approx}. 
\end{proof}
The following Lemma gives a bound on the estimation of the expected Risk with finite training samples:
\begin{lemma}[Bound on the Estimation Error] \label{lem:est} Fix $\delta >0$, then
$$\sup_{f\in \tilde{\mathcal{F}}} \left|\mathcal{E}_{V}(f)-\hat{\EE}_{V}(f)\right|\leq \frac{1}{\sqrt{N}}\left(4LsC+2V(0)+LC\sqrt{\frac{1}{2}\log \left(\frac{1}{\delta}\right)}\right),$$
with probability $1-\delta$.
\end{lemma}
\begin{proof}
The proof follows from Theorem \ref{theo:Rad} given in Appendix \ref{app:TT}. 
It is sufficient to bound the Rademacher complexity of the class $\tilde{\mathcal{F}}$:
\begin{align*}
\mathcal{R}_{N}(\tilde{\mathcal{F}})&=\mathbb{E}_{x,\sigma}\left[\sup_{f\in\tilde{\mathcal{F}}}\left|\frac{1}{N}\sum_{i=1}^N \sigma_i f(x_i)\right|\right]=\mathbb{E}_{x,\sigma}\left[\sup_{f\in\tilde{\mathcal{F}}}\left|\frac{s}{Nn}\sum_{i=1}^N \sigma_i\left(\sum_{j=1}^m \sum_{k=-n}^n w_{j,k}\hat{\psi}\left(x_i,t_j,\frac{sk}{n}\right)\right) \right|\right]\\
&= \mathbb{E}_{x,\sigma}\left[\sup_{f\in\tilde{\mathcal{F}}}\left|\frac{s}{Nn}\sum_{j=1}^m \sum_{k=-n}^n w_{j,k} \sum_{i=1}^N \sigma_i\hat{\psi}\left(x_i,t_j,\frac{sk}{n}\right) \right|\right]\\
&\leq \mathbb{E}_{x,\sigma} \frac{sC}{mNn}\sum_{j=1}^m \sum_{k=-n}^n\left|\sum_{i=1}^N \sigma_i\hat{\psi}\left(x_i,t_j,\frac{sk}{n}\right)\right| \text{ By Holder inequality: $\scalT{a}{b}\leq \nor{a}_{\infty}\nor{b}_{1}$}\\
&\leq \frac{sC}{mNn} \mathbb{E}_{x}\sum_{j=1}^m \sum_{k=-n}^n\sqrt{\mathbb{E}_{\sigma}\left(\sum_{i=1}^N \sigma_i\hat{\psi}\left(x_i,t_j,\frac{sk}{n}\right)\right)^2} \text{Jensen inequality, concavity of square root}
\end{align*} 
Note that $\mathbb{E}(\sigma_i\sigma_j)=0$, for $i \neq j$ it follows that:\\
$\mathbb{E}_{\sigma}\left(\sum_{i=1}^N \sigma_i\hat{\psi}\left(x_i,t_j,\frac{sk}{n}\right)\right)^2= \mathbb{E}_{\sigma} \sum_{i=1}^N \sum_{\ell=1}^N \sigma_i\sigma_{\ell}\hat{\psi}\left(x_i,t_j,\frac{sk}{n}\right)\hat{\psi}\left(x_{\ell},t_j,\frac{sk}{n}\right)= \sum_{i=1}^N\hat{\psi}^2\left(x_i,t_j,\frac{sk}{n}\right) \leq N$, since $\hat{\psi}(.,.,.)\leq 1$.
Finally: 
$$\mathcal{R}_{m}(\tilde{\mathcal{F}})\leq \frac{Cs}{\sqrt{N}}.$$
\end{proof}
\noindent We are now ready to prove Theorem \ref{theo:LearningInv}:
\begin{proof}[Proof of Theorem \ref{theo:LearningInv}] Let $f^*_{N}=\argmin_{f\in \tilde{\mathcal{F}}}\hat{\EE}_{V}(f)$,  $\tilde{f}=\argmin_{f\in \tilde{\mathcal{F}}}\mathcal{E}_{V}(f)$,
$f_{p}=\argmin_{f\in \mathcal{F}_{p}}\mathcal{E}_{V}(f)$.
\begin{align*}
\mathcal{E}_{V}(f^*_{N})-\min_{f\in \mathcal{F}_p}\mathcal{E}_{V}(f)&=\underbrace{\left(\mathcal{E}_{V}(f^*_{N})-\mathcal{E}_{V}(\tilde{f})\right)}_{\text{Statistical Error}}+\underbrace{\left(\mathcal{E}_{V}(\tilde{f})-\mathcal{E}_{V}(f_{p})\right)}_{\text{Approximation Error}}
\end{align*}
The first term is the usual estimation or statistical error than we can bound using Lemma \ref{lem:est}, we have:
\begin{align*}
\mathcal{E}_{V}(f^*_{N})-\mathcal{E}_{V}(\tilde{f})&=\left(\mathcal{E}_{V}(f^*_{N})- \hat{\EE}_{V}(f^*_{N})\right)+\underbrace{\left( \hat{\EE}_{V}(f^*_{N})-\hat{\EE}_{V}(\tilde{f})\right)}_{\leq 0,\text{by optimality of $f^*_{N}$}}+\left(\hat{\EE}_{V}(\tilde{f})-\EE_{V}(\tilde{f})\right) \\
&\leq 2\sup_{f\in \tilde{\mathcal{F}}} \left|\mathcal{E}_{V}(f)-\hat{\EE}_{V}(f)\right|\\
&\leq 2 \frac{1}{\sqrt{N}}\left(4LsC+2V(0)+LC\sqrt{\frac{1}{2}\log \left(\frac{1}{\delta}\right)}\right),
\end{align*}
with probability $1-\delta$ over the training samples.
Let $\tilde{f}_p$, the function defined in Lemma \ref{lem:approx}, that approximates $f_{p}$ in $\tilde{\mathcal{F}}$.  
By Lemma \ref{lem:comp} we know that:
$$\mathcal{E}_{V}(\tilde{f}_p)\leq \mathcal{E}_{V}(f_p)+  \frac{2sLC}{\sqrt{m}}\left(1+\sqrt{2\log\left(\frac{1}{\delta_1}\right)}\right)+L\left(\frac{2sC}{\sqrt{|G|}}\left( 1+\sqrt{2\log\left(\frac{m}{\delta_2} \right)}\right)+\frac{2sC}{n} \right), $$
with probability $1-\delta_1-\delta_2$, on the choice of the templates and the sampled group elements.
By optimality of $\tilde{f} \in \tilde{\mathcal{F}}$, we have $$\EE_{V}(\tilde{f})\leq \EE_{V}(\tilde{f}_p)\leq \mathcal{E}_{V}(f_p)+  \frac{2sLC}{\sqrt{m}}\left(1+\sqrt{2\log\left(\frac{1}{\delta_1}\right)}\right)+L\left(\frac{2sC}{\sqrt{|G|}}\left( 1+\sqrt{2\log\left(\frac{m}{\delta_2} \right)}\right)+\frac{2sC}{n} \right)$$
Hence by a union bound with probability $1-\delta-\delta_1-\delta_2$, on the training set , the templates and the group elements we have:
\begin{align*}
\mathcal{E}_{V}(f^*_{N})-\min_{f\in \mathcal{F}_p}\mathcal{E}_{V}(f) &\leq  2 \frac{1}{\sqrt{N}}\left(4LsC+2V(0)+LC\sqrt{\frac{1}{2}\log \left(\frac{1}{\delta}\right)}\right)\\
&+\frac{2sLC}{\sqrt{m}}\left(1+\sqrt{2\log\left(\frac{1}{\delta_1}\right)}\right)+L\left(\frac{2sC}{\sqrt{|G|}}\left( 1+\sqrt{2\log\left(\frac{m}{\delta_2} \right)}\right)+\frac{2sC}{n} \right).
\end{align*}
\end{proof}

%
%
%


\section{Technical tools}\label{app:TT}
\begin{theorem} \cite{dvoretzky1956}\label{theo:empCDF}
 Let $X_1, X_2, ..., X_m$ be i.i.d. random variables with cumulative distribution function $F$, and let $\hat{F}_m$ be the associated empirical cumulative density
function $\hat{F}_m=\frac{1}{m}\sum_{i=1}^m  \ind_{X_i\leq \tau}$. Then for any  $\gamma>0$
 $$\mathbb{P}\left\{\sup_{\tau}\left|\hat{F}_{m}(\tau)-F(\tau)\right|>\gamma\right\}\leq 2\exp\left(-2m\gamma^2\right).$$
\end{theorem}
\begin{lemma}[\cite{RahimiR08},Concentration of the mean of bounded random variables in a Hilbert Space] \label{lem:ConcH}Let $(\mathcal{H},\scalT{.}{.}_{\mathcal{H}})$ be a Hilbert space. Let $X_j$, $j=1\dots K$, be iid random, such that $||X_j||_{\mathcal{H}}\leq M$. Then for any $\delta>0$, with probability $1-\delta$,
$$\nor{ \frac{1}{K}\sum_{j=1}^K X_j - \frac{1}{K}\mathbb{E}\sum_{j=1}^K X_j}_{\mathcal{H}} \leq \frac{M}{\sqrt{K}}\left(1+\sqrt{2\log\left(\frac{1}{\delta}\right)}\right). $$ 
\end{lemma}

\begin{theorem}[\cite{RahimiR08}]\label{theo:Rad}
Let $\mathcal{F}$ be a bounded class of function, $\sup_{x\in\XX}\left|f(x)\right|\leq C$ for all $f\in \mathcal{F}$. Let $V$ be an $L$-Lipschitz loss. Then with probability $1-\delta$, with respect to training samples $\{x_i,y_i\}_{i=1\dots N}$,every $f$ satisfies:
$$\mathcal{E}_{V}(f)\leq \hat{\mathcal{E}}_{V}(f)+4L \mathcal{R}_{N}(\mathcal{F})+\frac{2V(0)}{\sqrt{N}}+LC\sqrt{\frac{1}{2N}\log\frac{1}{\delta}},$$
where $\mathcal{R}_{N}(\mathcal{F})$ is the Rademacher complexity of the class $\mathcal{F}$:
$$\mathcal{R}_{N}(\mathcal{F})=\mathbb{E}_{x,\sigma}\left[\sup_{f\in \mathcal{F}}\left|\frac{1}{N}\sum_{i=1}^N \sigma_i f(x_i)\right|\right],$$
the variables $\sigma_i$ are iid symmetric Bernoulli random variables taking value in $\{-1,1\}$, with equal probability and are independent form $x_i$.
\end{theorem}

\section{Numerical Evaluation}\label{app:TT}
\subsection{Permutation Invariance Experiment}

For our first experiment, we created an artificial dataset which was designed to exploit permutation invariance, providing us with a finite group to which we had complete access. The dataset $X_{perm}$ consists of all sequences of length $L=5$, where each element of the sequence is taken from an alphabet $A$ of 8 characters, giving us a total of 32,768 data points. Two characters $c_1, c_2\in A$ were randomly chosen and designated as targets, so that a sequence $x\in X_{perm}$ is labeled positive if it contains both $c1$ and $c_2$, where the position of these characters in the sequence does not matter. Likewise, any sequence that does not contain both characters is labeled negative. This provides us with a binary classification problem (positive sequences vs. negative sequences), for which the label is preserved by permutations of the sequence indices, i.e. two sequences will belong to the same orbit if and only if they are permuted versions of one another. \\
The $i^{\text{th}}$ character in $A$ is encoded as an 8-dimensional vector which is 0 in every position but the $i^{\text{th}}$, where it is 1. Each sequence $x\in X_{perm}$ is formed by concatenating the 5 such vectors representing its characters, resulting in a binary vector of length 40. To build the permutation-invariant representation, we project a binary sequences onto an equal-length sequence consisting of standard-normal gaussian vectors, as well as all of its permutations, and then pool over the projections with a CDF.\\
As a baseline, we also used a bag-of-words representation, where each $x\in X_{perm}$ was encoded with an 8-dimensional vector with $i^{\text{th}}$ element equal to the count of how many times character $i$ appears in $x$. Note that this representation is also invariant to permutations, and so should share many of the benefits of our feature map.\\
For all classification results, 4000 points were randomly chosen from $X_{perm}$ to form the training set, with an even split of 2000 positive points and 2000 negative points. The remaining 28,768 points formed the test set. \\
We know from Theorem 3 that the expected risk is dependent on the number of templates used to encode our data and on the number of bins used in the CDF-pooling step. The right panel of Figure \ref{fig:vsG} shows RLS classification accuracy on $X_{perm}$ for different numbers of templates and bins. We see that, for a fixed number of templates, increasing the number of bins will improve accuracy, and for a fixed number of bins, adding more templates will improve accuracy. We also know there is a further dependence on the number of transformation samples from the group $G$. The left panel of Figure \ref{fig:vsG} shows how classification accuracy, for a fixed number of training points, bins, and templates, depends on the number of transformation we have access to. We see the curve is rather flat, and there is a very graceful degradation in performance. \\
In Figure \ref{fig:RLSerr}, we include the sample complexity plot (for RLS) with the error bars added. 

\begin{figure*}[h]
\centerline{\epsfig{figure=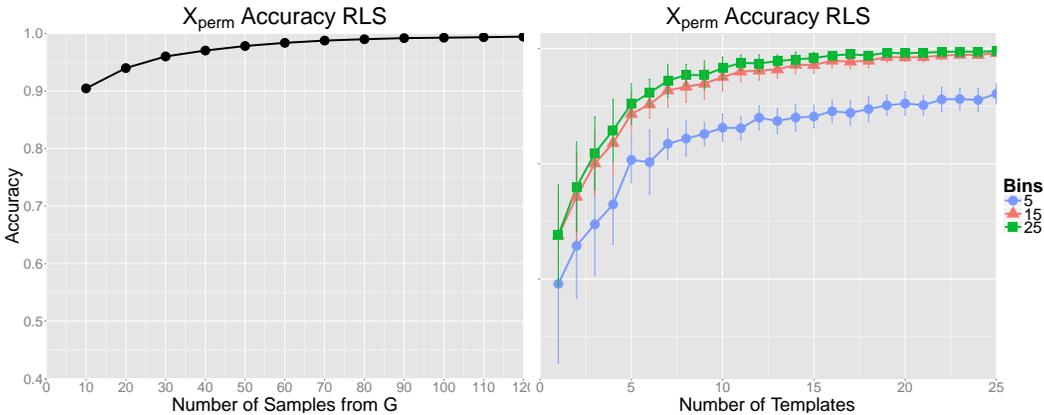, width=140mm, height=55mm}}
\caption{{\it Left) Classification accuracy of random invariant features as function of the number of sampled group elements on $X_{\text{perm}}$. Right) Classification accuracy of random invariant features as function of the number of templates and bin sizes on $X_{\text{perm}}$.}}
\label{fig:vsG}
\end{figure*}

\begin{figure*}[h]
\centerline{\epsfig{figure=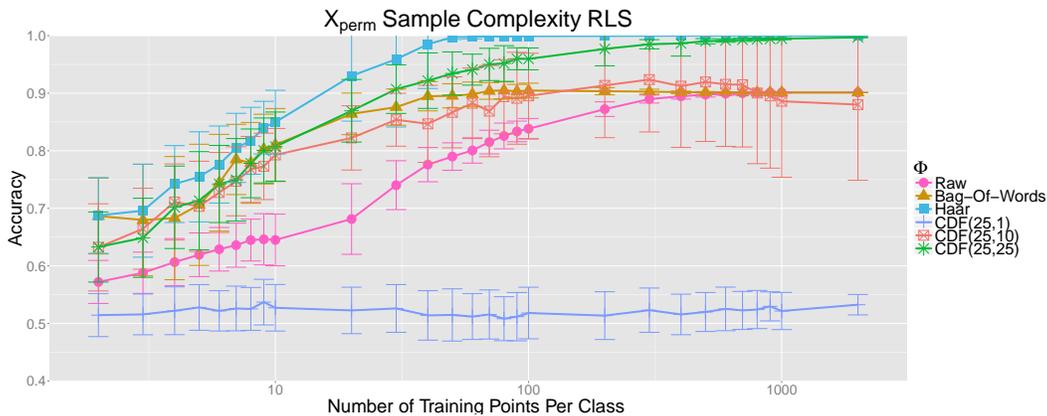, width=140mm, height=55mm}}
\caption{{\it 
Classification accuracy as a function of training set size. $\Phi = \text{CDF}(n,m)$ refers to a random feature map with $n$ bins and $m$ templates. For each training set size, the accuracy is averaged over 100 random training samples. With enough templates/bins, the random feature map outperforms the raw features as well as a bag-of-words representation (also invariant to permutation). We also train an RLS classifier with a haar-invariant kernel, which naturally gives the best performance. However, by increasing the number of templates, we come close to matching this performance with random feature maps.}}
\label{fig:RLSerr}
\end{figure*}

\subsection{TIDIGITS Experiment}

Here, we add plots (Figures \ref{fig:uttsupp},\ref{fig:womensupp} and \ref{fig:mensupp}) showing performance as a function of number of templates and bins for some other splits of the TIDIGITS data. 

\begin{figure*}[h]
\centerline{\epsfig{figure= 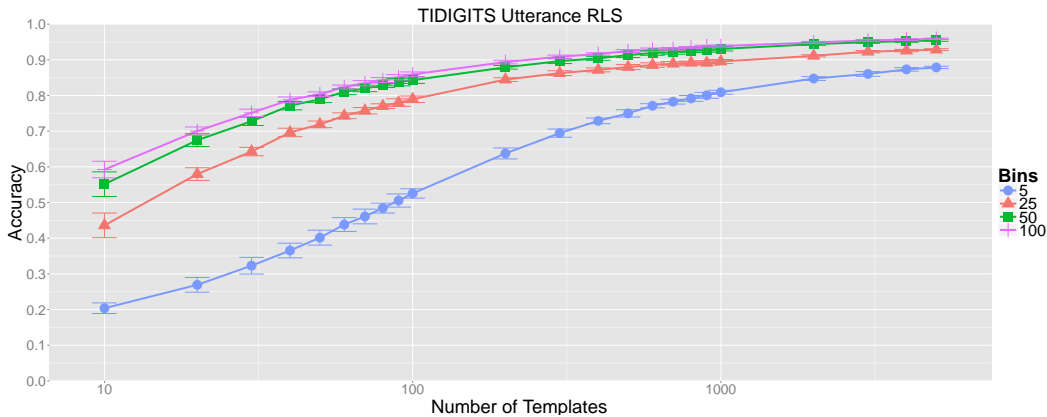, width=140mm, height=55mm}}
\caption{{\it Mean classification accuracy as a function of number of templates, $m$, and bins, $n$. Accuracy is averaged over 30 random template samples for each $m$ and error bars are displayed. In the ``Utterance'' dataset, we train and test on the same speakers, but the test set contains new utterances of each digit. This is the easiest dataset, representing only intraspeaker variability, and the performance is quite good even for a small number of bins. }}
\label{fig:uttsupp}
\end{figure*}

\begin{figure*}[h]
\centerline{\epsfig{figure= 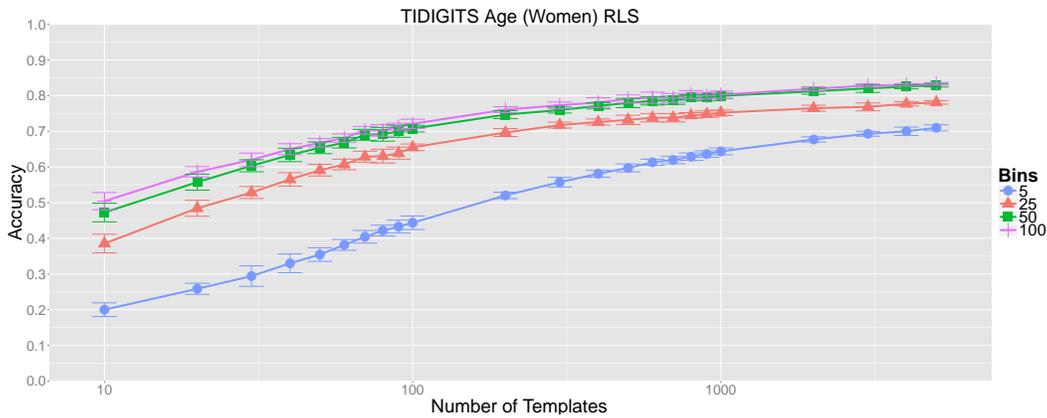, width=140mm, height=55mm}}
\caption{{\it Mean classification accuracy as a function of number of templates, $m$, and bins, $n$. Accuracy is averaged over 30 random template samples for each $m$ and error bars are displayed. In the ``Age (Women)'' dataset, we train on adult women and test on children, giving us an age mismatch. Despite this mismatch, performance remains strong.}}
\label{fig:womensupp}
\end{figure*}

\begin{figure*}[h]
\centerline{\epsfig{figure= 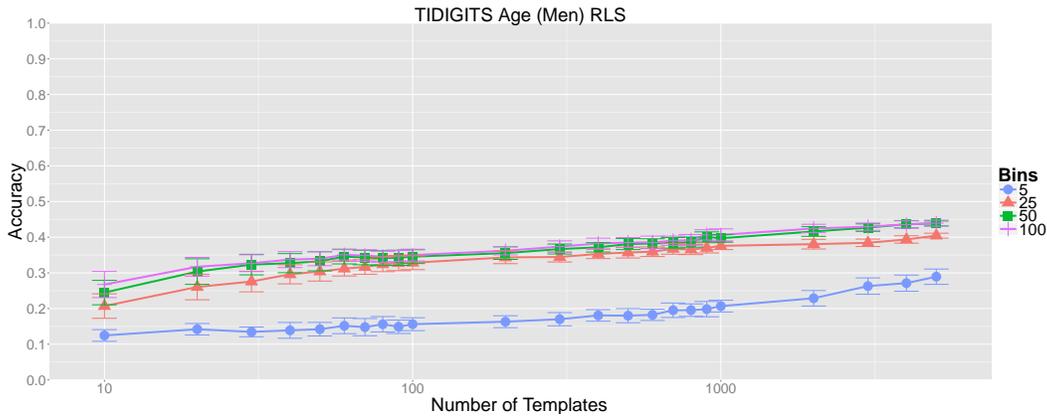, width=140mm, height=55mm}}
\caption{{\it Mean classification accuracy as a function of number of templates, $m$, and bins, $n$. Accuracy is averaged over 30 random template samples for each $m$ and error bars are displayed. In the ``Age (Men)'' dataset, we train on adult men and test on children, giving us an age mismatch. We see the weakest performance in this dataset, much worse than on the ``Age (Women)'' dataset. This is possibly due to the fact that women have higher pitched voices than men, creating less of a mismatch between women and children than men and children.}}
\label{fig:mensupp}
\end{figure*}


\end{document}